\documentclass[letterpaper, 10 pt, conference]{ieeeconf}  

\IEEEoverridecommandlockouts                              

\usepackage{multicol}
\usepackage{graphicx}
\usepackage[font=small,labelfont=bf,justification=justified]{caption}
\usepackage{subcaption}
\usepackage{amsmath}
\usepackage{amssymb}
\usepackage{textcomp}
\usepackage{gensymb}
\usepackage{bm}
\usepackage{float}
\usepackage{color}
\usepackage{multirow}
\usepackage{hhline}
\usepackage{changepage}
\usepackage{comment}
\usepackage{booktabs}
\usepackage{algorithm}
\usepackage{algpseudocode}
\usepackage{multirow}     
\usepackage{adjustbox}
\usepackage{siunitx}      
\usepackage{multirow}
\usepackage{cite}
\newtheorem{theorem}{Theorem}
\usepackage{mathtools}
\usepackage[hidelinks, bookmarks=true]{hyperref}

\usepackage[most]{tcolorbox}
\usepackage{lmodern}
\tcbset{
  systembox/.style={
    colback=blue!8,
    colframe=gray,
    boxrule=0.6pt,
    arc=2pt,
    width=\linewidth
  },
  userbox/.style={
    colback=gray!6,
    colframe=gray,
    boxrule=0.6pt,
    arc=2pt,
    width=\linewidth
  },
  assistantbox/.style={
    colback=orange!8,
    colframe=gray,
    boxrule=0.6pt,
    arc=2pt,
    width=\linewidth
  }
}

\title{\LARGE \bf
Extended LASMP: Language Aided Subset
Sampling Based Motion Planner for Drone
Navigation
}

\author{Saswati Bhattacharjee$^{1}$, Anirban Sinha$^{2}$, Mukulika Ghosh$^{3}$ and Chinwe Ekenna$^{1}$
\thanks{$^{1}$ University at Albany, Department of Computer Science, NY, USA. \texttt{\{sbhattacharjee, cekenna\}@albany.edu}, $^{2}$ GE Aerospace Research, NY, USA. \texttt{\{anirban.sinha1@geaerospace.com\}
},
$^{3}$ Missouri State University, Department of Computer Science,Missouri, USA.
\texttt{\{mghosh@missouristate.edu\}}
}
}

\begin{document}

\maketitle
\thispagestyle{empty}
\pagestyle{empty}

\begin{abstract}
This paper proposes a command-aligned sampling distribution that incorporates navigation commands directly into the exploration process to improve the efficiency of sampling-based motion planning for quadcopters in three-dimensional environments. Inspired by the Language-Aided Subset Sampling Motion Planner, this paper extends command-aware subset sampling to aerial systems.
 Subset sampling is generalized to command-aligned regions, enabling structured 3D exploration while preserving probabilistic completeness. Natural language instructions are grounded into yaw-conditioned and altitude-aware geometric constraints using a prompt-based large language model for entity extraction and instruction validation. Extensive simulations in complex environments demonstrate that the proposed method generates collision-free, consistent trajectories while significantly reducing sampling effort compared to uniform RRT-based planners.
\end{abstract}

\section{Introduction}\label{sec:introduction}
Autonomous aerial navigation in cluttered three-dimensional environments remains a fundamental challenge in robotics. Quadcopters operating in indoor and semi-structured environments must reason over higher-dimensional state spaces, handle limited sensing, and generate dynamically feasible trajectories while avoiding obstacles. Sampling-based planners such as Rapidly-exploring Random Trees (RRT)~\cite{lavalle1998rapidly,kuffner2000rrt}
and its optimal variant RRT$^*$~\cite{karaman2011sampling} have been widely adopted for such problems due to their probabilistic completeness and scalability to high-dimensional spaces. However, uniform sampling over large configuration spaces often results in poor sample efficiency, particularly in cluttered environments. Several works have attempted to improve sampling efficiency through heuristic or informed sampling strategies,
including Informed RRT$^*$ and related approaches~\cite{gammell2014informed,gammell2015batch,kuo2018deep}. Although these approaches bias exploration using geometric cost-to-go heuristics, they fail to incorporate structured task-level instructions into sampling, leading to inefficient exploration and wasted computation.

\begin{figure}[h!t]
    \centering
    \includegraphics[width=0.5\textwidth]{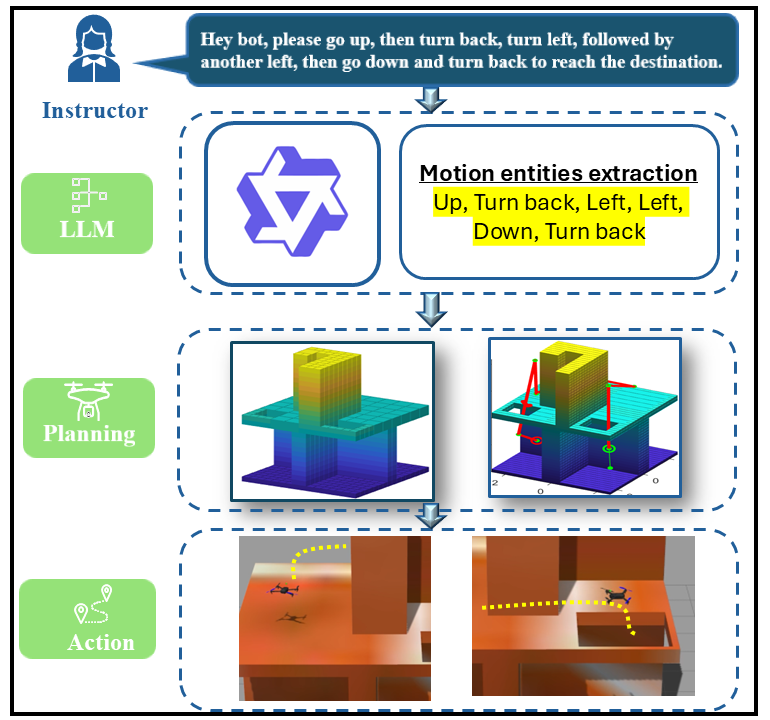}
    \caption{Extend LASMP: A drone navigating a multi-room environment based on natural language input while generating obstacle free paths on the fly. The red line shows the planned path while the yellow dotted lines show smoothened path executed on the robot.}
    \label{fig:fig_1}
\end{figure}

In many real-world applications, aerial robots are required to follow high-level human commands, such as \textit{go left}, \textit{move up}, or \textit{turn back and proceed forward}. Recent advances in large language models (LLMs) have enabled robots to interpret natural language instructions and map them to control or planning actions~\cite{ahn2022can,xie2023language,wu2024vision}. However, most existing approaches either rely on end-to-end learned policies or treat language as an external high-level planner, leaving the underlying sampling-based motion planner unchanged. In contrast, classical sampling-based planners ignore the directional intent embedded in commands, leading to unnecessary exploration of irrelevant regions of the planning space.

This work addresses the problem of \emph{command-aware sampling for quadcopter navigation}. Building upon the Language-Aided Subset Sampling Motion Planner (LASMP)\cite{bhattacharjee2025lasmp} developed for ground robots, we extend the framework to aerial platforms to plan in $\mathbb{R}^3 \times \mathbb{S}^1$. Unlike planar robots, quadcopters introduce an additional translational degree of freedom, requiring a transition from planar subset sampling to command-aligned three-dimensional subset sampling. Each navigation command is interpreted as a yaw update followed by directional translation or vertical motion, and the sampling distribution is dynamically restricted to subsets aligned with the commanded direction. This structured restriction concentrates exploration in more meaningful regions while preserving probabilistic completeness.

To robustly ground natural language instructions into geometric planning constraints, we integrate an LLM-based named entity recognition (NER) pipeline. Using few-shot prompt-based extraction, the system identifies task-relevant motion entities without requiring task-specific model retraining. The extracted entities are converted into yaw-conditioned and altitude-aware constraints
that directly shape the sampling subsets. Additionally, the LLM is employed to detect inconsistencies or ambiguities in user instructions and request corrective input, enhancing interaction reliability. An illustration of our work is shown in Fig. \ref{fig:fig_1} where a user can give the instruction and based on that instruction, our planner computes a feasible path.

The primary contributions of this paper are as follows:
\begin{itemize}
    \item A unified command-aware sampling framework that integrates directional navigation commands into the sampling distribution of a sampling-based planner together with a ray-guided switching strategy for consistent sequential command execution.
  \item Comprehensive empirical validation demonstrating improved sample efficiency, reduced computational cost, faster convergence compared to uniform sampling planners across diverse 3D environments.
    \item A complete ROS2–PX4–Gazebo integration of the framework, enabling physics-based execution and validating the real-time feasibility of command-aware subset sampling on a quadcopter flight stack.
      \item A \textit{probabilistic completeness analysis} for command-restricted sampling, extending classical RRT guarantees to structured subset sampling.
\end{itemize}

\section{Related Work}\label{sec:related_works}

\subsection{Sampling-Based Motion Planning}
Sampling-based motion planning methods such as Rapidly-exploring Random Trees (RRT)~\cite{lavalle1998rapidly} 
and its bidirectional variant RRT-Connect~\cite{kuffner2000rrt}
have been widely adopted for high-dimensional planning problems due to their probabilistic completeness and computational efficiency. Probabilistic Roadmaps (PRM)~\cite{kavraki2002probabilistic}
provide a multi-query alternative with similar theoretical guarantees.
Karaman and Frazzoli~\cite{karaman2011sampling} later introduced RRT$^*$ and PRM$^*$,
establishing asymptotic optimality within the sampling-based planning framework.
Additional planners such as FMT$^*$~\cite{janson2015fast} and BIT$^*$~\cite{gammell2015batch}
further improve efficiency through graph-based search and batch sampling.
Despite these advances, classical sampling-based planners typically rely on uniform sampling over the configuration space,
which can lead to poor sample efficiency in large or cluttered environments.

\subsection{Informed and Biased Sampling}
To improve exploration efficiency, several works restrict sampling to heuristically defined subsets of the planning space. Informed RRT$^*$~\cite{gammell2014informed} samples from a prolate hyperspheroid around a feasible solution, focusing search on regions that can improve path cost. Dynamic-domain RRT~\cite{yershova2005dynamic} reduces sampling near obstacles to improve convergence in cluttered spaces. While these methods bias sampling using geometric or cost-to-go heuristics, they do not incorporate structured task instructions or semantic constraints to draw samples. Our approach restricts sampling based on \emph{command-induced directional constraints}. Rather than relying on cost heuristics or solution refinement, we align sampling subsets with yaw-updated motion commands before finding a feasible solution. This enables semantically guided exploration and preserves probabilistic completeness.

\subsection{Kinodynamic and UAV Motion Planning}

For aerial robots, kinodynamic constraints must be considered during planning. Closed-loop RRT (CL-RRT)~\cite{luders2010bounds} integrates system dynamics into the planning loop and provides tracking guarantees. Other works address quadrotor motion planning in 3D environments using trajectory optimization or motion primitives~\cite{mellinger2011minimum,richter2016polynomial}. While these approaches account for dynamic feasibility, they typically sample uniformly in the state space or rely on optimization-based methods, without incorporating structured navigation commands into the sampling strategy. Our work focuses on the geometric planning layer in $\mathbb{R}^3 \times \mathbb{S}^1$, introducing command-aware subset sampling as a lightweight yet theoretically grounded mechanism for improving exploration efficiency in aerial navigation tasks.

\subsection{Instruction-Conditioned Planning}
Recent research has explored grounding natural language instructions into robotic decision-making. Approaches such as SayCan~\cite{ahn2022can,chen2024spatialvlm} and language-conditioned path planning frameworks~\cite{xie2023language, anderson2018vision, tellex2011understanding, hu2019safe} use large language models to map instructions to feasible actions.
Vision-language navigation methods~\cite{wu2024vision,shah2021ving} further integrate perception and language for navigation. However, many of these methods rely on learned policies or treat language as an external task planner.

In this work, we preserve a classical sampling-based planning backbone that integrates language-derived constraints directly into the geometric sampling distribution. This allows us to maintain theoretical guarantees such as probabilistic completeness while enabling structured exploration for the navigation.

\section{Problem Statement}\label{sec:problem_statement}
Let $D \subset \mathbb{R}^3 \times \mathbb{S}^1$ denote the planning space for a quadcopter,
    where a state $\bm{x}=(\bm{p},\psi)$ consists of position $\bm{p} \in \mathbb{R}^3$
and yaw $\psi \in \mathbb{S}^1$.
Let $D_{\mathrm{obs}} \subset D$ be the obstacle of the environment, and define the collision-free set as
\begin{equation}
    D_{\mathrm{free}} = D \setminus D_{\mathrm{obs}}
\end{equation}
Given a start state $\bm{x}_s \in D_{\mathrm{free}}$, a goal state $\bm{x}_g \in D_{\mathrm{free}}$, and an ordered sequence of navigation commands
\begin{equation}
    \Phi = [\phi_1, \phi_2, \ldots, \phi_K]
\end{equation}
where each $\phi_j$ belongs to the command set
\begin{equation} \label{eq:cmd_list}
    \begin{aligned}
  \{ & \texttt{STRAIGHT}, \texttt{LEFT}, \texttt{RIGHT}, \texttt{UP}, \texttt{DOWN}, \\
     & \texttt{TURN BACK}, \texttt{LEFT INPLACE}, \texttt{RIGHT INPLACE} \}
    \end{aligned}
\end{equation}
The objective is to compute a collision-free trajectory
\begin{eqnarray} \nonumber
    & \sigma : [0,1] \rightarrow D_{\mathrm{free}} \\
    \text{s.t.} & \quad \sigma(0) = \bm{x}_s \qquad \sigma(1) = \bm{x}_g
\end{eqnarray}
Each command $\phi_j$ imposes a motion constraint on the trajectory. Commands $\{\texttt{LEFT}, \texttt{RIGHT}\}$ correspond to a yaw update after a translation along the updated heading, while $\texttt{STRAIGHT}$ corresponds to forward motion without yaw change. Commands $\{\texttt{UP}, \texttt{DOWN}\}$ correspond to pure vertical motion. While $\{\texttt{TURN BACK}, \texttt{LEFT INPLACE}, \texttt{RIGHT INPLACE}\}$ correspond to yaw update without translation. The trajectory must execute the commands in order.

\noindent
Uniform sampling over $D$ ignores the directional structure induced by $\Phi$ and leads to excessive exploration of irrelevant regions. The problem addressed in this work is therefore to design a sampling-based planner that restricts sampling to dynamically constructed subsets
\begin{equation}
    D_{\mathrm{smp}}(\bm{x},\phi_j) \subset D
\end{equation}
aligned with the commanded motion, while preserving probabilistic completeness.

\section{Methodology}\label{sec:methodology}

Fig. \ref{fig:workflow} illustrates the workflow of the proposed planner. The method consists of two modules: a language module and a planner module. The language module is responsible for language grounding, while the planner module computes the path from the start location to the goal location based on the natural language instruction.
\begin{figure}[!t]
    \centering
    \includegraphics[width=0.5\textwidth]{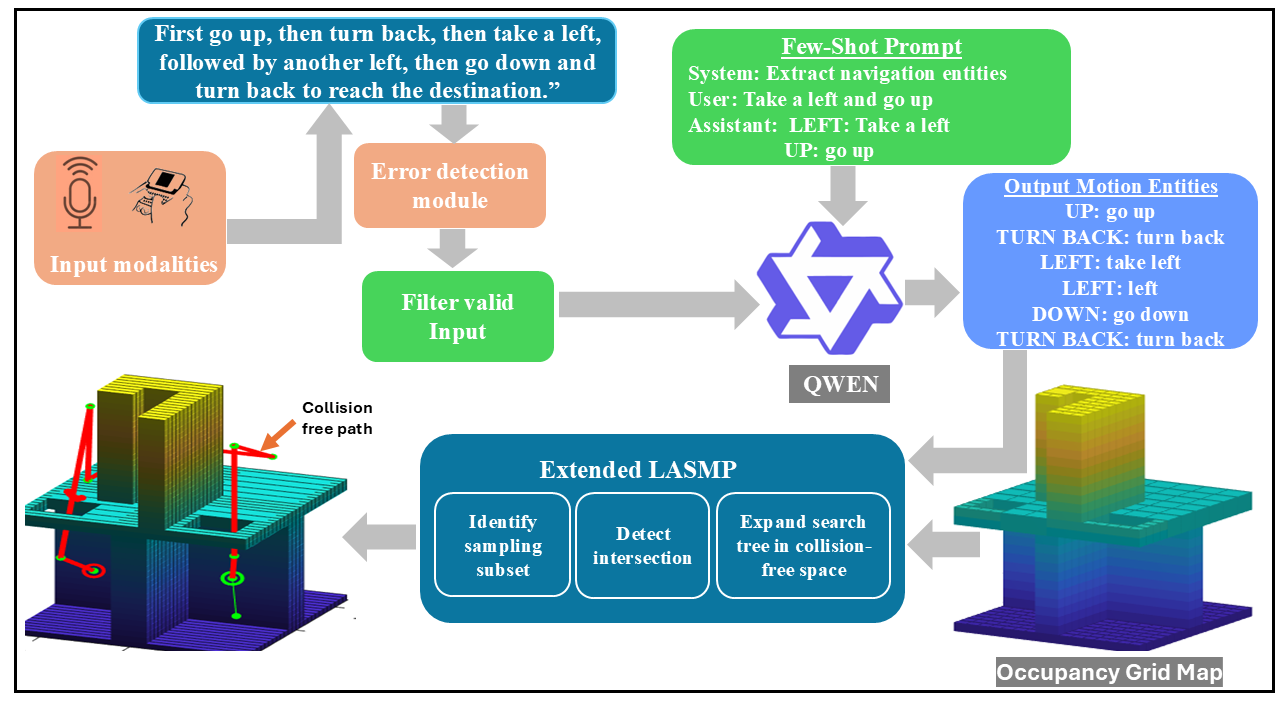}
    \caption{(a) Architectural overview of the Extended LASMP pipeline: The Qwen~\cite{yang2025qwen3} model process the natural language instructions to extract a sequence of directional primitives (e.g., \texttt{UP}, \texttt{LEFT}). This turn list serves as a heuristic constraint for the Extend LASMP framework, which optimizes the search space toward a task-relevant subset of the planning  domain.}
    
    \label{fig:workflow}
\end{figure}

We propose a command-aware sampling-based planner presented in Algorithm \ref{alg:ext_lasmp} that incrementally constructs a search tree in the state space
$D \subset \mathbb{R}^3 \times \mathbb{S}^1$ while respecting the ordered command sequence $\Phi = [\phi_1,\ldots,\phi_K]$. 
The algorithm integrates command-aligned subset sampling with ray-based switching during incremental tree expansion.
The key idea is to restrict sampling to dynamically defined subsets $D_{\mathrm{smp}}(\bm{x},\phi_j) \subset D$ that are aligned with the motion induced by the current command $\phi_j$. This focuses exploration on relevant regions of the environment while preserving probabilistic completeness. The planner maintains a tree $T \subset D_{\mathrm{free}}$ rooted at $\bm{x}_s$ and grows as valid nodes become available. 

\subsection{Command Interpretation}
Each command $\phi_j$ induces a desired motion direction.
Let $\psi(\bm{x})$ denote the yaw component of state $\bm{x}$.

For a given command $\phi_j$ the yaw is updated as,
\begin{equation}
    \psi_j = \psi(\bm{x}) + \Delta\psi(\phi_j)
\end{equation}
where
\begin{equation}
    \Delta\psi(\phi_j) \in \left\{0,\ \tfrac{\pi}{2},\ -\tfrac{\pi}{2},\ \pi \right\}
\end{equation}
Note that the commands $\{\texttt{UP}, \texttt{DOWN}\}$ correspond to pure vertical motion. Following a yaw update the resulting unit heading vector is defined as
\begin{equation}
    \mathbf{v}_j =
\begin{cases}
(\cos\psi_j,\ \sin\psi_j,\ 0), & \text{horizontal motion}\\[4pt]
(0,0,\pm 1), & \text{vertical motion}
\end{cases}
\end{equation}

\subsection{Command-Aligned Sampling Subsets}
To expand node $\bm{x} \in T$ for command $\phi_j$, samples are drawn from the subset centered at $\bm{p}(\bm{x})$ and oriented along $\mathbf{v}_j$.
\begin{equation}
    D_{\mathrm{smp}}(\bm{x},\phi_j) \subset D
\end{equation}

Let $(h,w,d)$ denote the forward, lateral, and vertical extents of the subset, then it can be defined as
\begin{equation}
    C(\bm{x},\phi_j) =
     \bm{p} \in \mathbb{R}^3 :
    \begin{cases}
    0 \le \mathbf{v}_j^\top (\bm{p} - \bm{p}(\bm{x})) \le h\\
    \|(I - \mathbf{v}_j \mathbf{v}_j^\top)(\bm{p} - \bm{p}(\bm{x}))\| \le w\\
    |z - z(\bm{x})| \le d
    \end{cases}
\end{equation}

The full subset is then
\begin{equation}
    D_{\mathrm{smp}}(\bm{x},\phi_j) = C(\bm{x},\phi_j) \times \{\psi_j\}
\end{equation}
This construction restricts sampling to states that are forward-aligned with the commanded motion.

\subsection{Tree Expansion}
At each iteration, the planner selects an expansion node $\bm{x} \in T$ and constructs the corresponding command-aligned sampling subset $D_{\mathrm{smp}}(\bm{x},\phi_j)$. A sample $\bm{x}_{\mathrm{rand}}$ is drawn uniformly from this subset, and the planner attempts to extend the tree toward it using a steering function with maximum step size $\eta$, i.e., $\bm{x}_{\mathrm{new}} = \textsc{Steer}(\bm{x}, \bm{x}_{\mathrm{rand}}, \eta)$. If the local path between $\bm{x}$ and $\bm{x}_{\mathrm{new}}$ lies entirely within the collision-free set $D_{\mathrm{free}}$, the new state is added to the tree as a child of $\bm{x}$. Repeating this procedure incrementally expands the tree within command-aligned regions of the state space.

\subsection{Command Execution and Switching}
Commands are executed sequentially according to the ordered list $\Phi = [\phi_1,\ldots,\phi_K]$ where the command transitions are triggered by the detection of an unobstructed region in the direction of the next commanded motion. During each extension from $\bm{x}$ toward a sampled state $\bm{x}_{\mathrm{rand}}$, a ray is cast from the new state $\bm{x}_{\mathrm{new}}$ in the direction corresponding to the subsequent command $\phi_{j+1}$. Let $\mathbf{v}_{j+1}$ denote the unit vector associated with $\phi_{j+1}$. The ray is defined as
\begin{equation}
    \bm{r}(t) = \bm{p}(\bm{x}_{\mathrm{new}}) + t\,\mathbf{v}_{j+1}, \qquad t \ge 0
\end{equation}
If the ray does not intersect any obstacle within a prescribed look-ahead distance $L$, i.e.,
\begin{equation}
    \bm{r}([0,L]) \subset D_{\mathrm{free}}
\end{equation}
the planner identifies the presence of an opening in the environment in the direction of the next command. In this case, the current command $\phi_j$ is considered complete, and the planner switches to command $\phi_{j+1}$. This strategy enables the planner to detect navigable corridors, intersections, or free-space expansions. The switching condition depends on the local visibility of free space in the commanded direction, making it suitable for cluttered environment. Section~\ref{sec:prob_complete} outlines the assumptions of the Extended LASMP and provides a detailed proof of its probabilistic completeness property. 

\begin{algorithm}[t]
\caption{Extended LASMP
\\ \textbf{Input:} $\mathcal{M}$(map), $\bm{x}_s=(\bm{p}_s,\psi_s)$, $\bm{x}_g$, $\Phi$, $\eta$, $h,w,d$, $L$
\\ \textbf{Output:} Collision-free trajectory $\tau$ }
\label{alg:ext_lasmp}
\begin{algorithmic}[1]
\State $T\gets\{\bm{x}_s\}$; $\textsc{Parent}(\bm{x}_s)\gets\varnothing$; $\bm{x}_{\text{curr}}\gets \bm{x}_s$; $j\gets 1$
\While{$j\le K$ \textbf{and} $\|\bm{p}(\bm{x}_{\text{curr}})-\bm{p}(\bm{x}_g)\|>\epsilon_g$}
    \State $\psi_j \gets \textsc{wrap}(\psi(\bm{x}_{\text{curr}})+\textsc{YawIncrement}(\phi_j))$
    \State $\mathbf{v}_j \gets \textsc{CommandDirection}(\phi_j,\psi_j)$
    \If{$j<K$}
        \State $\psi_{j+1}\gets \textsc{wrap}(\psi_j+\textsc{YawIncrement}(\phi_{j+1}))$
        \State $\mathbf{v}_{j+1}\gets \textsc{CommandDirection}(\phi_{j+1},\psi_{j+1})$
    \EndIf

    \State $D_{\text{smp}} \gets \textsc{BuildCuboidSubset}(\bm{x}_{\text{curr}},\mathbf{v}_j,h,w,d)$
    \State $\bm{x}_{\text{rand}} \gets \textsc{SampleUniform}(D_{\text{smp}})$
    \State $\psi(\bm{x}_{\text{rand}})\gets \psi_j$
    \State $\bm{x}_{\text{near}} \gets \textsc{SelectNode}(T, \bm{x}_{\text{rand}})$

    \State $\bm{x} \gets \bm{x}_{\text{near}}$
    \For{$s=0$ \textbf{to} $\eta$ \textbf{step} $\Delta s$}
        \State $\bm{x}' \gets \textsc{SteerStep}(\bm{x},\bm{x}_{\text{rand}},\Delta s,\psi_j)$
        \If{\textbf{not} \textsc{CollisionFree}$(\bm{x},\bm{x}',\mathcal{M})$} \State \textbf{break} \EndIf
        \State $\bm{x} \gets \bm{x}'$
        \If{$j<K$ \textbf{and} \textsc{RayFree}$(\bm{p}(\bm{x}),\mathbf{v}_{j+1},L)$}
            \State \textbf{break} \Comment{Opening found for next cmd}
        \EndIf
    \EndFor

    \If{$\bm{x} \neq \bm{x}_{\text{near}}$}
        \State $T\gets T\cup\{\bm{x}\}$; $\textsc{Parent}(\bm{x})\gets \bm{x}_{\text{near}}$; $\bm{x}_{\text{curr}}\gets \bm{x}$
    \EndIf
    \If{$j<K$ \textbf{and} \textsc{RayFree}$(\bm{p}(\bm{x}_{\text{curr}}),\mathbf{v}_{j+1},L)$}
        \State $j \gets j+1$
    \EndIf
\EndWhile

\If{$\|\bm{p}(\bm{x}_{\text{curr}})-\bm{p}(\bm{x}_g)\|\le \epsilon_g$}
    \State $\tau \gets \textsc{ExtractPath}(T,\bm{x}_{\text{curr}})$; \Return $\tau$
\EndIf
\end{algorithmic}
\end{algorithm}

\subsection{Natural Language-Based Instruction Grounding}
Given a navigation instruction represented as a token sequence 
\(
\mathcal{S} = \{w_1, w_2, \dots, w_n\},
\)
the input is first processed by an error detection module 
\(
\mathcal{C}(\mathcal{S})
\)
that filters malformed commands. The errors are categorized into three types: misleading entities, typographical errors, and ambiguous directions. Misleading entities occur when motion-related entities are interpreted differently from their intended meaning (e.g., the word left is used in a context other than indicating a direction). Typographical errors arise from misspelled words, while ambiguous directions include instructions such as 'Go over there and continue moving,' which do not specify a clear destination or direction. In all of the above scenarios, the model asks the user to provide the correct input. Only valid instructions 
\(
\mathcal{S}^{*} = \mathcal{C}(\mathcal{S})
\)
are forwarded to the entity extraction stage. The task is formulated as conditional sequence generation, where a pretrained autoregressive language model produces a structured output sequence \( u \) conditioned on the validated instruction and few-shot demonstrations \( \mathcal{D}_{\text{few}} \):
\begin{equation}
P(u \mid \mathcal{S}^{*}, \mathcal{D}_{\text{few}}) 
= \prod_{m=1}^{T1} 
P(u_{m} \mid u_{<m}, \mathcal{S}^{*}, \mathcal{D}_{\text{few}})
\end{equation}

where $T_1$ denotes the length of the generated token sequence $u$. The generated sequence \( u \) is parsed to obtain a set of entity–label pairs 
\begin{equation}
\mathbb{E} = \{(e_i, c_i)\}_{i=1}^{k}, \quad k=|E| \quad \text{no. of pairs}
\end{equation}
where \( e_i \subseteq \mathcal{S}^{*} \) denotes a contiguous text and \( c_i \) its corresponding entity type. The extracted entities are mapped to navigation primitives $\Phi$ for path planning (more details in section~\ref{sec:results}.B). As shown in Fig.~\ref{fig:workflow}, the user instruction is first processed by an error detection module and then passed to the Qwen model for motion-related entity extraction. We employ a ChatML-based few-shot prompting strategy for motion-related named entity recognition, as illustrated in Fig.~\ref{fig:PromptEngineering}.

\begin{figure}[ht]
\centering

\begin{tcolorbox}[systembox, title=\textbf{System}]
You are an expert in Named Entity Recognition. Extract all specified entity types from the provided text.
\end{tcolorbox}

\vspace{1pt}

\begin{tcolorbox}[userbox, title=\textbf{User}]
Extract turn list such as: LEFT, RIGHT, STRAIGHT, UP, DOWN, TURN BACK

\textit{Text:} ``Please head up, then turn back and go down to move to the destination.''
\end{tcolorbox}

\vspace{1pt}

\begin{tcolorbox}[assistantbox, title=\textbf{Assistant}]
UP: head up \\
TURN BACK: turn back \\
DOWN: go down
\end{tcolorbox}

\caption{An illustration of Chat Markup Language (ChatML)few-shot prompt to extract the motion entities. Messages are structured by roles such as system, user, and assistant. The system message defines the model’s behavior, the user message defines the task, and the assistant message can include an example output.}

 \label{fig:PromptEngineering}
\end{figure}



\section{Results}\label{sec:results}

\subsection{Experimental Setup}
The path was computed on an Intel i7 machine (16 GB RAM, NVIDIA GeForce RTX GPU).
The GPU was utilized by the language module and the planning module was implemented on the CPU. The Qwen3-1.7B large language model was loaded from the Hugging Face model repository using the AutoModelForCausalLM and AutoTokenizer classes provided by the Transformers library to implement the language module. Fig. \ref{fig:pointcloud_env} shows all environments utilized in experiment. As shown in the Fig. \ref{fig:pointcloud_env}, three two-story building environments Env-1, Env-2, and Env-3 and one single-story environment, Env-4, were considered for the experiment. 
 
\begin{figure}[ht]
    \centering
    \includegraphics[width=0.4\textwidth]{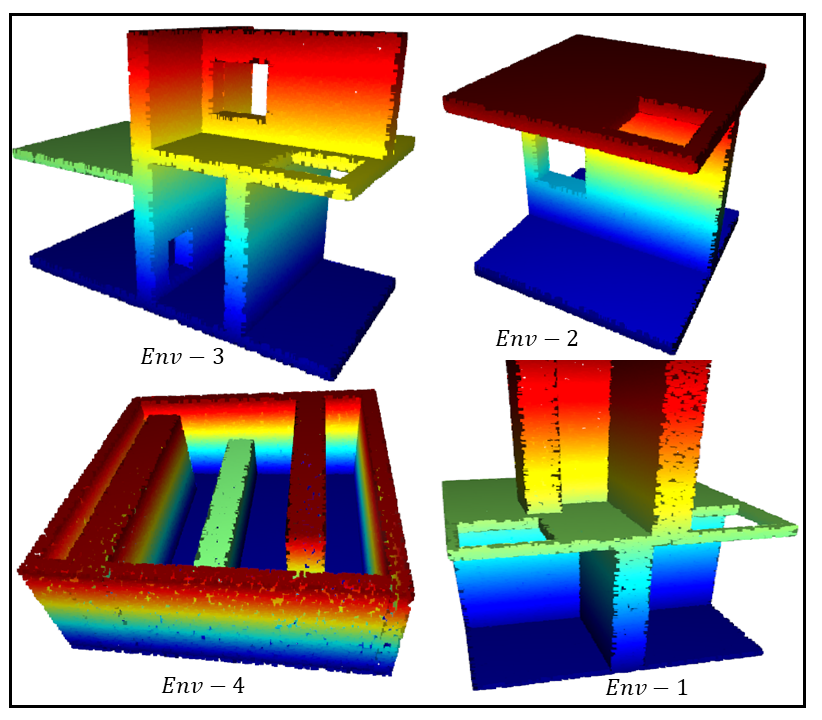}
    \caption{Diverse 3D point cloud environments for experimental validation. The environments vary in obstacle density, and corridor width to evaluate the robustness of command-aware subset sampling under different structural complexities.}
    \label{fig:pointcloud_env}
\end{figure}
\vspace{-2mm}
\subsection{Natural Language grounding}
To illustrate the effectiveness of the NER task, consider the instruction :
\begin{quote}
\texttt{Hey Bot, go up, then take left, followed by a right to enter hallway.}
\end{quote}

After validation through the filtering module 
\(
\mathcal{S}^{*} = \mathcal{C}(\mathcal{S}),
\)
the instruction is provided to the language model along with the few-shot demonstrations 
\(
\mathcal{D}_{\text{few}},
\) 
the model generates an output sequence
\(
u \sim \mathbb{P}(u \mid \mathcal{S}^{*}, \mathcal{D}_{\text{few}}),
\)
which for the given case yields a continuous string:
\begin{eqnarray} \nonumber
u = 
\{
    \texttt{UP:} \texttt{go up},
    \texttt{LEFT:} \texttt{take left}, \\ \nonumber
    \texttt{RIGHT:} \texttt{followed by a right}\}
\end{eqnarray}

The sequence is parsed to obtain structured entity set
\begin{eqnarray} \nonumber
    \mathbb{E} = \{(\text{``go up''}, \texttt{UP}),
    (\text{``take left''}, \texttt{LEFT}),\\ \nonumber
     (\text{``followed by a right''}, \texttt{RIGHT})
    \}
\end{eqnarray}
These entity--label pairs are subsequently mapped to navigation primitives
\(
f:\mathbb{E} \rightarrow \Phi,
\)
resulting in the executable action sequence:
$\Phi = \{\texttt{UP},\; \texttt{LEFT},\; \texttt{RIGHT}\}$.
The action sequence $\Phi$ can be utilized by the planner. We implemented a voice interface where spoken commands are transcribed using the Whisper~\cite{radford2023robust} speech recognition model and then processed by the language understanding module described in Section~\ref{sec:methodology}. The performance of Qwen was evaluated on a custom dataset containing 100 erroneous navigation instructions. The model identified the navigation entities with 75\% accuracy
\vspace{-2mm}
\subsection{Performance Evaluation of Extended LASMP}
The \textit{Extended LASMP}, is evaluated against RRT, RRT* and RRT Connect using uniform sampling strategy to draw random samples.
\begin{figure}[ht]
    \centering
    \fbox{%
        \begin{minipage}{0.95\linewidth}
            \centering
            
            \includegraphics[width=\linewidth]{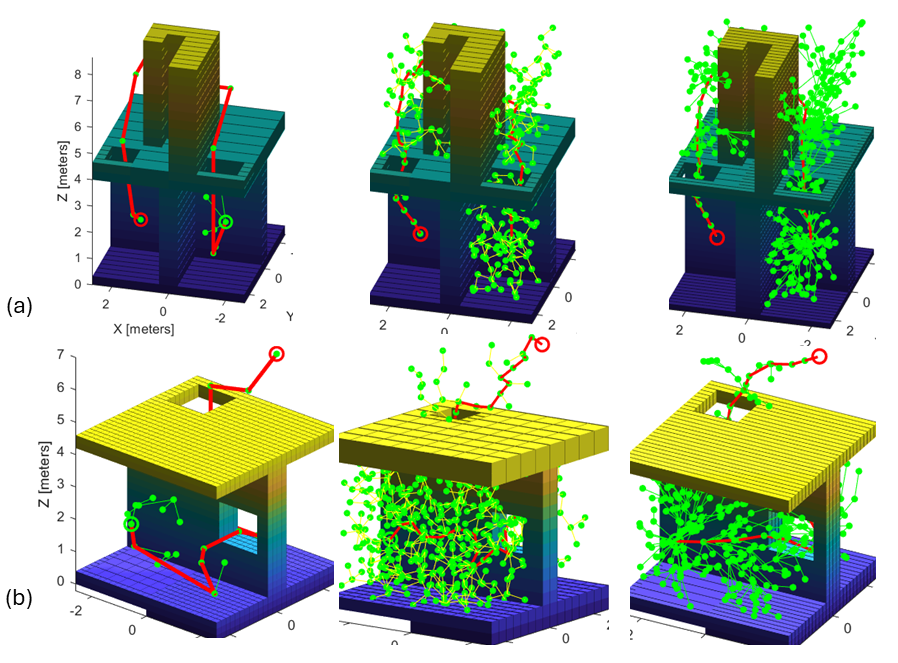}
            
            \includegraphics[width=\linewidth]{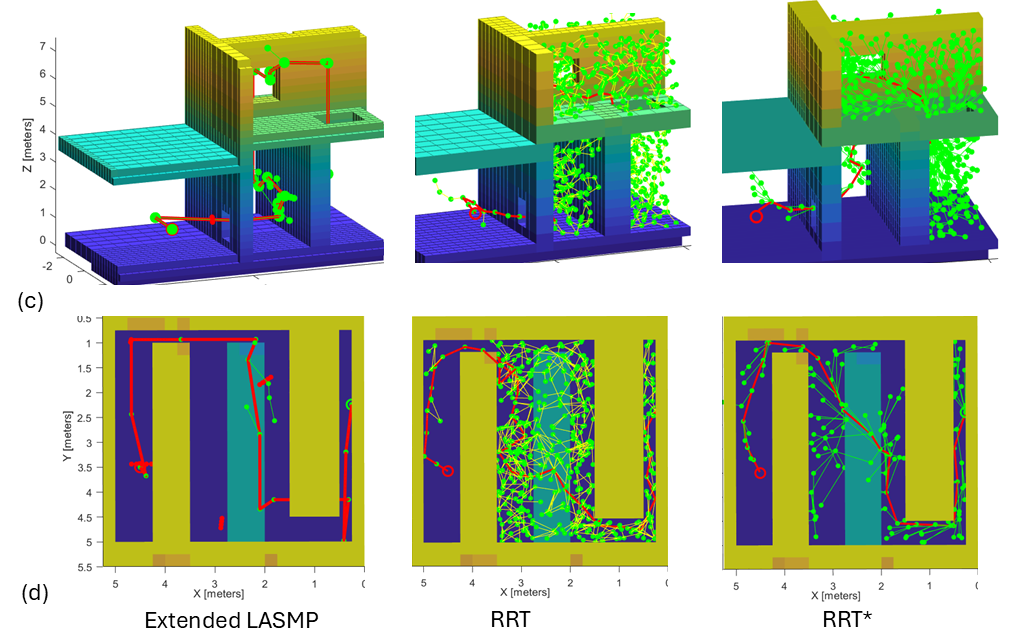}
            
        \end{minipage}
    }
    \caption{A comparative evaluation of the planners Extended LASMP, RRT, and RRT* in terms of the number of nodes added to the search tree. The red line denotes the collision-free path extracted from the search tree. Extended LASMP needed significantly less number of nodes as compared to RRT and RRT*.}
    \label{fig:grouped}
\end{figure}
\begin{table}[!t]
\centering
\caption{Quantitative comparison across four environments. 
Lower values are better for all metrics ($\downarrow$). 
Results are reported over 10 runs and executed in i7 processor in Ubuntu 24.04 LTS.}
\label{tab:performance_comparison}
\begin{adjustbox}{width=0.5\textwidth}
\begin{tabular}{|llcccc|}
\toprule
Environments & Methods & Nodes & Queries  & Time [s]  & Path Length [m]  \\
\midrule

\multirow{3}{*}{Env1} 
 & Extended LASMP & $\mathbf{8}\pm 2.09$ & $\mathbf{12} \pm 4.6$ & $\mathbf{2.361} \pm 0.73 $  & $9.95 \pm  0.76$ \\
 & RRT      & $ 491 \pm 155.71 $ & $ 1546 \pm 550.31$ & $2.618\pm0.77$  & $14.57\pm1.70$ \\
 & RRT*     & $493\pm142.24$  & $ 1356 \pm 326.91 $ & $ 15.095 \pm 5.28$  & $\mathbf{9.18 } \pm 0.77 $\\
 & RRTConnect &$ 140 \pm 42.41 $ & $749 \pm 282.05$ &$2.398 \pm 0.95$ & $20.53 \pm 1.12$\\
\midrule

\multirow{3}{*}{Env2} 
 & Extended LASMP & $\mathbf{9}\pm6.21$ & $\mathbf{11} \pm 6.10$& $ 1.825 \pm 1.40$  & $11.18 \pm 1.27$  \\
 & RRT      & $ 252 \pm 60.85 $ & $ 701 \pm 289.76 $ & $ 3.501 \pm 1.06 $ & $12.97 \pm 1.60 $ \\
 & RRT*     & $ 401 \pm 144.07$  & $ 1069\pm 426.26$ & $11.603 \pm 6.13$  & $\mathbf{10.41} \pm 0.82$\\
 & RRTConnect & $186\pm44.90$ & $455\pm149.09$ & $\mathbf{1.541}\pm0.36$ & $19.20\pm2.63$\\
\midrule

\multirow{3}{*}{Env3} 
 & Extended LASMP & $\mathbf{14}\pm8.37$ & $\mathbf{19}\pm8.36$  & $6.418\pm3.94$  & $15.76\pm1.57$  \\
 & RRT      & $1042\pm159.93$   & $3879\pm549.13$   & $8.368\pm0.93$  & $20.652\pm2.15$  \\
 & RRT*     & $664\pm300.82$  & $2430\pm940.28$  & $23.581\pm12.44$  & $\mathbf{12.972}\pm1.16$  \\
 & RRTConnect & $425\pm73.94$ & $1715\pm309.90$ & $\mathbf{4.933}\pm1.14$ & $26.14\pm2.63$\\
\midrule

\multirow{3}{*}{Env4} 
 & Extended LASMP & $\mathbf{20 }\pm7.48$& $\mathbf{35}\pm11.49$ & $7.61\pm4.34$  & $14.561\pm1.70$ \\
 & RRT     &  $289\pm111.29$ & $1748\pm540.61$ & $5.531\pm1.45$  & $14.187\pm0.65$  \\
 & RRT*     &  $204\pm73.77$ & $1055\pm354.41$ & $13.316\pm6.26$ & $\mathbf{11.211}\pm 0.27$  \\
 & RRTConnect & $68\pm10.02$ & $537\pm193.42$ & $\mathbf{2.131}\pm 0.78$ & $14.965\pm0.93$ \\
\bottomrule
\end{tabular}
\end{adjustbox}
\end{table}
\begin{figure}[!t]
    \centering
    \includegraphics[width=0.5\textwidth]{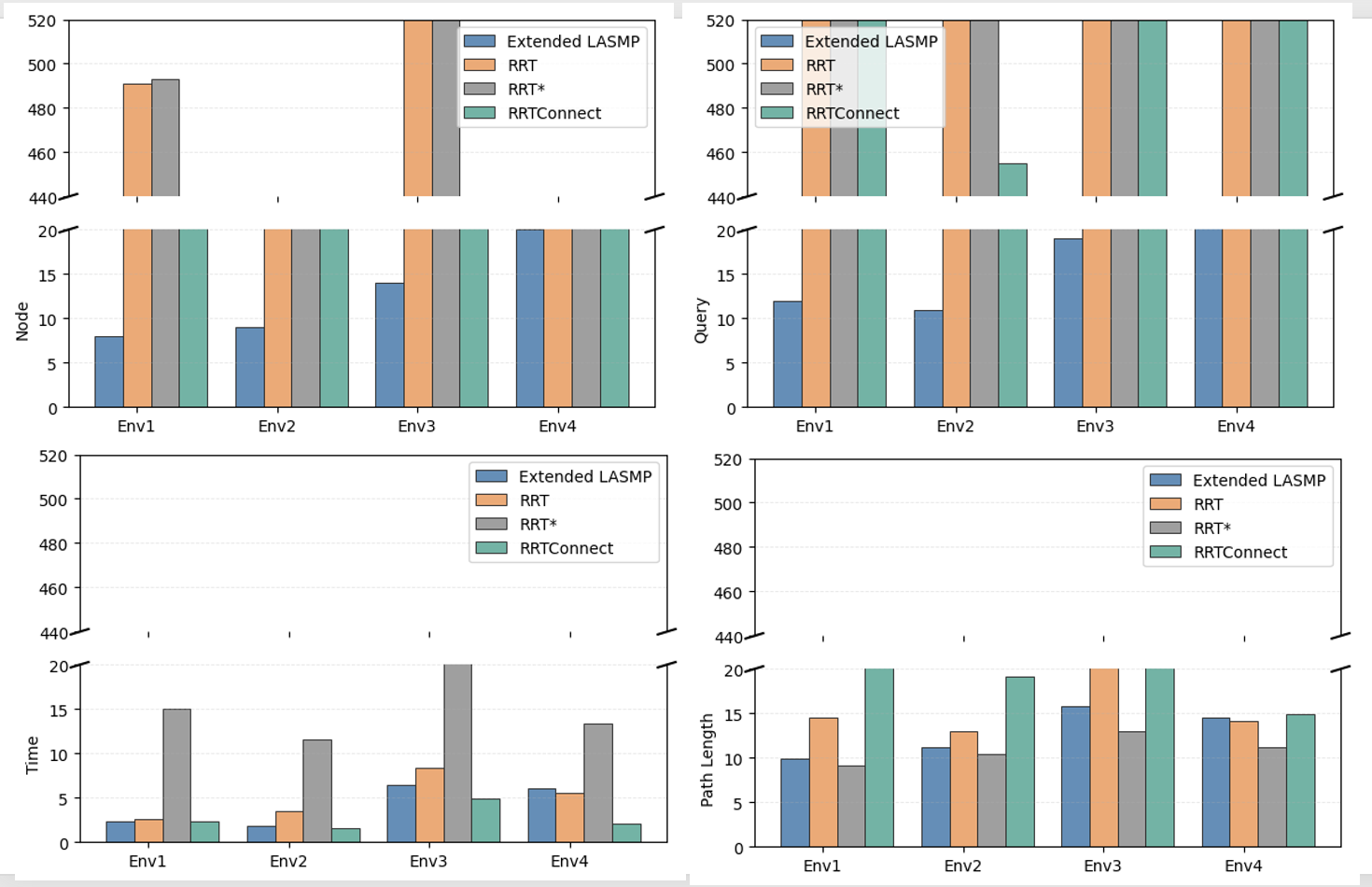}
    \caption{Performance comparison of Extended LASMP, RRT, RRT* and RRTConnect w.r.t number of nodes added to the search tree, the number of random-state queries, the computation time, and PathLength. 
    }
    \label{fig:comparison_plot}
\end{figure}
The  results are reported in Table \ref{tab:performance_comparison} and the computed paths for each environment are illustrated in Fig. \ref{fig:grouped}. The comparison across various metrics is reported on Fig. \ref{fig:comparison_plot}. The corresponding start and goal locations for each environment and the turn lists extracted from user instructions are reported in Table \ref{tab:fourcol}. Table \ref{tab:planner_parameters} shows the parameters of the algorithms.

\begin{table}[ht]
\centering
\caption{The start and goal configurations with corresponding turn lists used to evaluate the performance of Extended LASMP.}
\label{tab:fourcol}
\begin{adjustbox}{width=0.5\textwidth}
\begin{tabular}{|lccp{3cm}|}
\toprule
Environments & Start[m] & Goal[m] & Turn list \\
\midrule
Env1 & [-1.5 0.75 2.0 ]  & [1.33 1.33 2.0] & [UP, TURN BACK, LEFT, LEFT, DOWN, TURN BACK] \\
\midrule
Env2 & [-1.5 -1.5 1.8] & [1.5 1.25 7.0] & [LEFT , LEFT, UP, TURN BACK] \\
\midrule
Env3 & [-1.5 4.0 2.5 ] & [2.0 -2.75 1.5 ] & [UP , RIGHT IN PLACE, RIGHT, DOWN, TURN BACK, RIGHT]\\
\midrule
Env4 & [0.25 2.25 0.3] & [4.5 3.5 0.3 ] & [RIGHT, RIGHT, LEFT, LEFT] \\
\bottomrule
\end{tabular}
\end{adjustbox}
\end{table}

\begin{table}[t]
\centering
\caption{Planning parameters used for each algorithm.}
\label{tab:planner_parameters}
\resizebox{\linewidth}{!}{%
\begin{tabular}{|lcccc|}
\hline
\textbf{Parameter} & \textbf{Extended LASMP} & \textbf{RRT} & \textbf{RRT*} & \textbf{RRT-Connect} \\
\hline
Step size (m)         & 0.1 & 0.5 & 0.4 & 0.5 \\
Maximum iterations & 200 & 5000 & 4000 & 5000 \\
Goal tolerance(m)     & 2 & N/A & N/A & N/A \\
Goal bias          & N/A & 0.1 & 0.1 & 0.1 \\
Rewiring (m)          & N/A & N/A & 1 & N/A \\
\hline
\end{tabular}}
\end{table}

All results are averaged over $10$ independent runs to ensure statistical consistency. 
Across the four environments, Extended LASMP reduced the number of nodes by approximately 93.1\%–98.7\%, 90.2\%–98.4\%, and 70.6\%–96.7\% compared with RRT, RRT*, and RRT-Connect, respectively. Similarly, it reduced the number of random-state queries by approximately 98.0\%–99.5\%, 96.7\%–99.2\%, and 93.5\%–98.9\%, respectively.
A comparable trend is observed in computation time and the path length.  
Although Extended LASMP occasionally produces paths that are slightly longer, approximately 2\% longer than RRT and 20\% longer than RRT* this is consistent with the objective of the proposed framework. 
Our method achieves the lowest computation time in Env1. In the remaining environments, its computation time is slightly higher than RRT-Connect. 
Fig.~\ref{fig:comparison_plot} illustrates the comparison between our method, RRT, RRT*, RRT-Connect. As shown in the node and query plots of the figure, Extended LASMP requires significantly fewer nodes and queries compared to the baseline sampling-based algorithms. In terms of computation time and path length, our method is comparable to RRT, RRT* and RRT-Connect. The primary goal of our planner is not strict path optimality, but adherence to user-provided navigation instructions. Consequently, when a user specifies a longer route for moving one location to another, our planner follows that trajectory even if a shorter path exists. Overall, Extended LASMP achieves significant improvements in sampling efficiency while maintaining competitive computational performance relative to the baseline planners. 

\subsection{Simulation Setup and Execution in ROS2 and Gazebo}
We validated the proposed Extended LASMP in ROS2-Gazebo simulation environment. Static indoor environments with cluttered obstacles as shown in Fig.~\ref{fig:pointcloud_env} are considered. The quadcopter dynamics is simulated using the default ODE physics engine in Gazebo, while PX4~\cite{px4_ros2_user_guide} is employed as the low-level flight controller for attitude stabilization and waypoint tracking. The planner generates waypoints in $\mathbb{R}^3 \times \mathbb{S}^1$, which are subsequently smoothed and interpolated before execution. The planner is implemented as a node that publishes a sequence of waypoints as \texttt{geometry\_msgs/PoseStamped} messages. The setpoints are tracked through a position control interface at 50~Hz, ensuring dynamic feasibility.

Opening detection is implemented by querying the occupancy map along a ray in the direction of the next command. A command switch is triggered if no obstacle is detected within a predefined look-ahead distance. Figure~\ref{fig:gz_sim_result} shows the action sequence of the quadcopter for the user command: \textit{Go up, then turn back, take a left followed by another left, then go down and turn back to reach your destination}. Key evaluation metrics are summarized in Table~\ref{tab:performance_comparison}, while the total execution time in Gazebo is $23.2$~s. The simulation results confirm that command-aware sampling produces dynamically executable trajectories and integrates seamlessly with the ROS2–PX4 flight control stack. If the extracted turn list is infeasible, the planner reverts to the standard RRT for planning.

\begin{figure}
    \centering
    \includegraphics[width=1.03\linewidth]{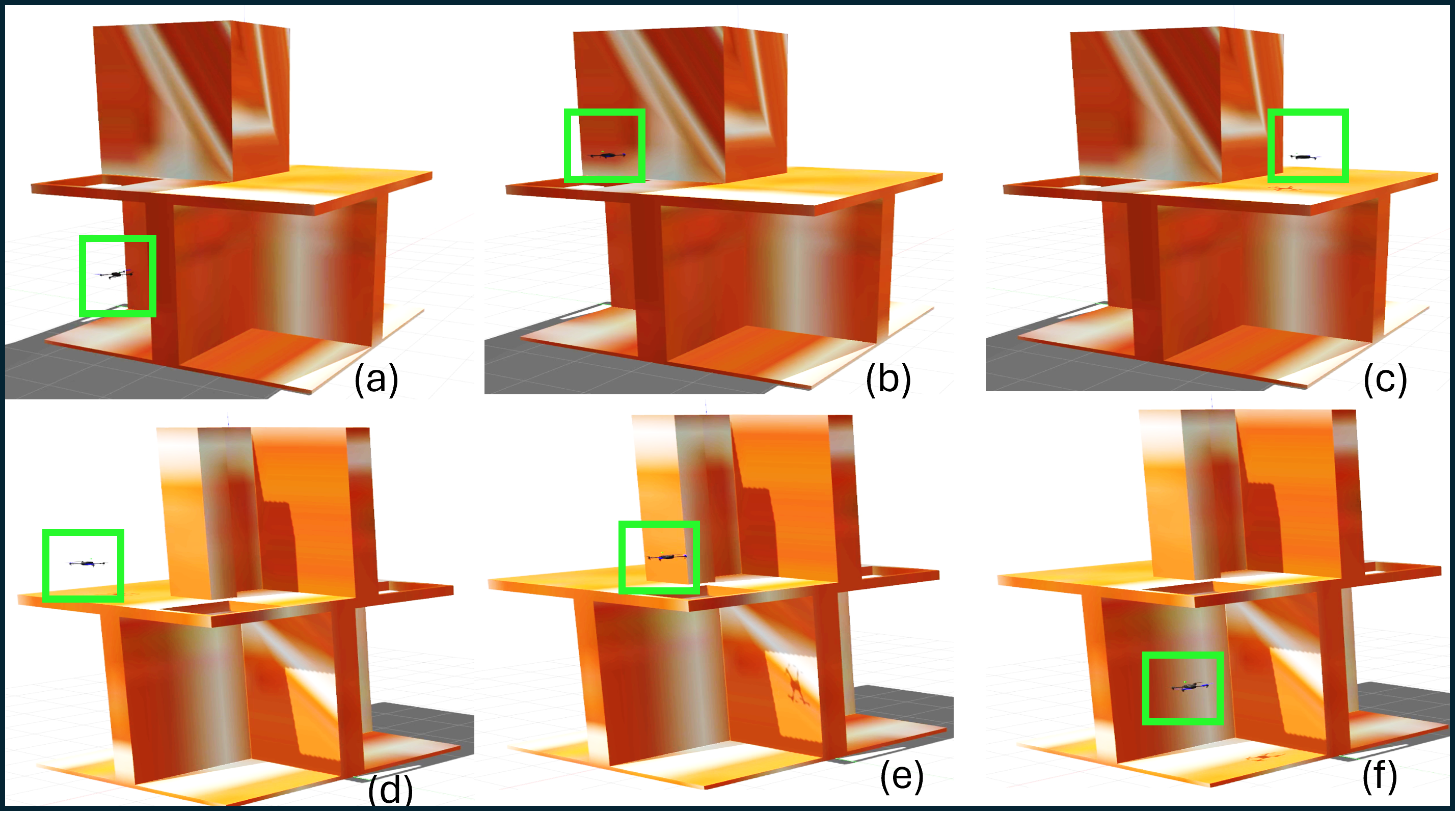}
    \caption{
    Instruction sequence \textit{Go up, then turn back, take a left followed by another left, then go down and turn back to reach your destination}. The snapshots illustrate yaw-conditioned subset sampling and ray-based switching between commands, resulting in a collision-free trajectory. The green blocks indicate the quadcopter poses at various time sequences.}
    \label{fig:gz_sim_result}
\end{figure}

\subsection{Ablation Study} \label{sec:ablation}

To evaluate the influence of the LLM component in the proposed framework Extended LASMP, we conducted an ablation study in which the LLM module was bypassed and the turn list was provided as a hard-coded input. The experimental results indicate no observable difference in performance with respect to the number of queries, the number of nodes added to the search tree, computation time, or the resulting path length. These observations suggest that the LLM in our framework functions solely as an interface for extracting navigation instructions and does not influence the underlying path computation process. As part of the ablation study as shown in Table \ref{tab:ablation}, we replaced the guided sampling strategy with a randomly oriented subset sampling approach. 
\begin{table}[!ht]
\centering
\caption{Ablation study comparing Extended LASMP with a randomly oriented subset-sampling planner in Environment~1.}
\label{tab:ablation}
\begin{adjustbox}{width=0.5\textwidth}
\begin{tabular}{|llcccc|}
\toprule
Environment & Method & Nodes & Queries  & Time [s]  & Path Length [m]  \\
\midrule

\multirow{2}{*}{Env1} 
 & Extended LASMP & \textbf{8} & \textbf{12} & \textbf{2.361}  & \textbf{9.95}  \\
 & Random subset sampling      & 54 & 240 & 153  & 11.53 \\

\bottomrule
\end{tabular}
\end{adjustbox}
\end{table}
The randomly oriented subset sampling planner was still able to find a feasible path; however, the convergence time increased significantly (approximately 68 times higher), and the search tree contained a substantially larger number of nodes (approximately 7 times more). Also, it required about 20 times more queries as compared to Extended LASMP. Although the resulting path length of the randomly oriented subset sampling planner is comparable, Extended LASMP outperformed across all other metrics. These results indicate that while random sampling preserves path feasibility, the proposed guided sampling strategy significantly improves search efficiency by reducing unnecessary exploration.

\section{Conclusions}\label{sec:conclusion}

We utilized a large language model to extract the motion-related entities, which serves as a key step of guiding our planner to efficiently compute the path.
While sampling-based planners like RRT and RRT* may achieve shorter paths in simple environments, they frequently fail to converge in complex scenarios involving narrow passages. In contrast, our Extended LASMP approach demonstrates superior reliability and command-following efficiency. By prioritizing successful instruction execution over path optimality, Extended LASMP proves to be a robust planner for autonomous navigation.

Future work will focus on deploying the planner on a hardware system to evaluate performance under sensing uncertainties. Another direction is to incorporate dynamic obstacle handling and online replanning to support navigation in partially known or changing environments. Additionally, integrating richer semantic instructions and multi-step task reasoning could further enhance the capability of language-guided aerial navigation.

\section{Appendix}\label{sec:appendix}
\subsection{Probabilistic Completeness of Extended LASMP} \label{sec:prob_complete}
To prove the probabilistic completeness property of the Extended LASMP planner, let us assume that the collision-free space $D_{\mathrm{free}} \subset D$ is open and measurable, and that there exists a continuous feasible path $\sigma:[0,1]\to D_{\mathrm{free}}$ from $\bm{x}_s$ to $\bm{x}_g$ with strong $\delta$-clearance for some $\delta>0$, i.e., $\mathcal{B}(\sigma(t),\delta)\subset D_{\mathrm{free}}$ for all $t\in[0,1]$. We assume that the steering function can connect any two states whose distance is at most $\eta$ whenever the local path between them lies entirely in $D_{\mathrm{free}}$. Furthermore, we assume that along the feasible path there exists a finite sequence of milestones such that whenever the tree reaches a neighborhood of one milestone, the corresponding command-aligned sampling subset contains a neighborhood of the next milestone. Finally, we assume that the sampling distribution over each subset $D_{\mathrm{smp}}(\bm{x})$ has density bounded below by a positive constant, ensuring that every measurable region within the subset has nonzero probability of being sampled.

\begin{theorem}
\textit{If a feasible path exists from $\bm{x}_s$ to $\bm{x}_g$, the planner finds a collision-free path with probability approaching $1$ as the number of iterations $N\to\infty$}.
\end{theorem}

\begin{proof}
\textit{Path covering:} Let $\delta>0$ denote the clearance of the feasible path $\sigma$ and let $\eta>0$ be the steering step size.
Choose a radius
\begin{equation}
    r \in \Bigl(0,min\{\tfrac{\delta}{2},\tfrac{\eta}{3}\}\Bigr)
\end{equation}
Since $\sigma$ is continuous on $[0,1]$, there exists a finite sequence of points $\bm{y}_0,\ldots,\bm{y}_m$ on $\sigma$ with $\bm{y}_0=\bm{x}_s$, $\bm{y}_m=\bm{x}_g$, and $\|\bm{y}_{i+1}-\bm{y}_i\|\le r$ $\forall i$.
Define $R_i \triangleq \mathcal{B}(\bm{y}_i,r)$. Because $2r<\delta$, we have $\mathcal{B}(\bm{y}_i,2r)\subset D_{\mathrm{free}}$ $\forall i$.

\textit{Positive progress probability: }
Assume that the tree $T$ contains a node in $R_i$. By the subset coverage assumption, there exists an expansion node $\bm{x} \in R_i$ such that the next milestone region $R_{i+1}$ is contained in the command-aligned sampling subset $D_{\mathrm{smp}}(\bm{x})$. Let $\mu(\cdot)$ denote the Lebesgue measure on the continuous components of the state space. By the assumption that the sampling density on each subset is bounded below by a constant $\lambda>0$, for any measurable set $A \subset D_{\mathrm{smp}}(\bm{x})$
\begin{equation}
    \mathbb{P}(\bm{x}_{\mathrm{rand}}\in A) \ge \lambda\,\frac{\mu(A)}{\mu(D_{\mathrm{smp}}(\bm{x}))}
\end{equation}

Applying this to $A=R_{i+1}$ yields
\begin{equation}
    \mathbb{P}(\bm{x}_{\mathrm{rand}}\in R_{i+1}) \ge \lambda\,\frac{\mu(R_{i+1})}{\mu(D_{\mathrm{smp}}(\bm{x}))} \ge c > 0
\end{equation}
the constant $c$ is independent of $i$ since $R_{i+1}$ has fixed positive measure, $D_{\mathrm{smp}}(\bm{x})$ has uniformly bounded volume. If $\bm{x}\in R_i$ and $\bm{x}_{\mathrm{rand}}\in R_{i+1}$, by triangle inequality,
\begin{multline}
    \|\bm{x}_{\mathrm{rand}}-\bm{x}\| \le \|\bm{x}_{\mathrm{rand}}-\bm{y}_{i+1}\| \\
    + \|\bm{y}_{i+1}-\bm{y}_i\| + \|\bm{y}_i-\bm{x}\| \le 3r \le \eta
\end{multline}

\noindent
Since $R_i$ and $R_{i+1}$ lie inside the $\delta$-clear tube around $\sigma$, the local connection between $\bm{x}$ and $\bm{x}_{\mathrm{rand}}$ remains in $D_{\mathrm{free}}$. By the steering assumption, planner can add a node in $R_{i+1}$. Thus, each iteration advances from $R_i$ to $R_{i+1}$ with probability at least $q>0$ for some constant $q$.

\textit{Convergence: }
The probability of failing to reach $R_{i+1}$ after $N$ iterations is $\le (1-q)^N \to 0$. Since the number of regions is finite, the probability of reaching $R_m$ (and hence $\bm{x}_g$) approaches $1$ as $N\to\infty$.
\end{proof}

\subsubsection{Corollary (Mixture Sampling)}
If samples are drawn from
\begin{equation}
    \bm{x}_{\mathrm{rand}} \sim
    \begin{cases}
    \mathrm{Uniform}(D), & \text{w.p. } \varepsilon\\
    \mathrm{Uniform}(D_{\mathrm{smp}}(\bm{x})), & \text{w.p. } 1-\varepsilon
    \end{cases}
\end{equation}
with $\varepsilon>0$, every open subset of $D_{\mathrm{free}}$
has nonzero sampling probability.
Hence probabilistic completeness follows from standard RRT assumptions.

\bibliographystyle{IEEEtran}

\bibliography{thebibliography}

@article{lavalle1998rapidly,
  title={Rapidly-exploring random trees: A new tool for path planning},
  author={LaValle, Steven},
  journal={Research Report 9811},
  year={1998},
  publisher={Department of Computer Science, Iowa State University}
}

@inproceedings{kuffner2000rrt,
  title={RRT-connect: An efficient approach to single-query path planning},
  author={Kuffner, James J and LaValle, Steven M},
  booktitle={Proceedings 2000 ICRA. Millennium Conference. IEEE International Conference on Robotics and Automation. Symposia Proceedings (Cat. No. 00CH37065)},
  year={2000},
  organization={IEEE}
}

@article{karaman2011sampling,
  title={Sampling-based algorithms for optimal motion planning},
  author={Karaman, Sertac and Frazzoli, Emilio},
  journal={The international journal of robotics research},
  volume={30},
  number={7},
  pages={846--894},
  year={2011},
  publisher={Sage Publications Sage UK: London, England}
}

@inproceedings{gammell2014informed,
  title={Informed RRT: Optimal sampling-based path planning focused via direct sampling of an admissible ellipsoidal heuristic},
  author={Gammell, Jonathan D and Srinivasa, Siddhartha S and Barfoot, Timothy D},
  booktitle={2014 IEEE/RSJ International Conference on Intelligent Robots and Systems},
  year={2014},
  organization={IEEE}
}

@inproceedings{gammell2015batch,
  title={Batch informed trees (BIT*): Sampling-based optimal planning via the heuristically guided search of implicit random geometric graphs},
  author={Gammell, Jonathan D and Srinivasa, Siddhartha S and Barfoot, Timothy},
  booktitle={ICRA},
  pages={3067--3074},
  year={2015},
  organization={IEEE}
}

@article{ahn2022can,
  title={Do as i can, not as i say: Grounding language in robotic affordances},
  author={Ahn, Michael and Brohan, Anthony and Brown, Noah and Chebotar, Yevgen and Cortes, Omar and David, Byron and Finn, Chelsea and Fu, Chuyuan and Gopalakrishnan, Keerthana and Hausman, Karol and others},
  journal={arXiv preprint arXiv:2204.01691},
  year={2022}
}

@inproceedings{xie2023language,
  title={Language-conditioned path planning},
  author={Xie, Amber and Lee, Youngwoon and Abbeel, Pieter and James, Stephen},
  booktitle={Conference on Robot Learning},
  pages={3384--3396},
  year={2023},
  organization={PMLR}
}

@article{wu2024vision,
  title={Vision-language navigation: a survey and taxonomy},
  author={Wu, Wansen and Chang, Tao and Li, Xinmeng and Yin, Quanjun and Hu, Yue},
  journal={Neural Computing and Applications},
  volume={36},
  number={7},
  pages={3291--3316},
  year={2024},
  publisher={Springer}
}

@inproceedings{bhattacharjee2025lasmp,
  title={LASMP: Language Aided Subset Sampling Based Motion Planner},
  author={Bhattacharjee, Saswati and Sinha, Anirban and Ghosh, Mukulika and Ekenna, Chinwe},
  booktitle={2025 IEEE 21st Intl. Conf. on Automation Science and Engineering (CASE)},
  pages={1632--1639},
  year={2025},
  organization={IEEE}
}

@article{kavraki2002probabilistic,
  title={Probabilistic roadmaps for path planning in high-dimensional configuration spaces},
  author={Kavraki, Lydia E and Svestka, Petr and Latombe, J-C and Overmars, Mark H},
  journal={IEEE transactions on Robotics and Automation},
  volume={12},
  number={4},
  pages={566--580},
  year={2002},
  publisher={IEEE}
}

@article{janson2015fast,
  title={Fast marching tree: A fast marching sampling-based method for optimal motion planning in many dimensions},
  author={Janson, Lucas and Schmerling, Edward and Clark, Ashley and Pavone, Marco},
  journal={The Intl. journal of robotics research},
  volume={34},
  number={7},
  pages={883--921},
  year={2015},
  publisher={SAGE Publications Sage UK: London, England}
}

@inproceedings{yershova2005dynamic,
  title={Dynamic-domain RRTs: Efficient exploration by controlling the sampling domain},
  author={Yershova, Anna and Jaillet, L{\'e}onard and Sim{\'e}on, Thierry and LaValle, Steven M},
  booktitle={Proc. of the 2005 IEEE intl. conf. on robotics and automation},
  pages={3856--3861},
  year={2005},
  organization={IEEE}
}

@inproceedings{luders2010bounds,
  title={Bounds on tracking error using closed-loop rapidly-exploring random trees},
  author={Luders, Brandon D and Karaman, Sertac and Frazzoli, Emilio and How, Jonathan P},
  booktitle={Proceedings of the 2010 american control conference},
  pages={5406--5412},
  year={2010},
  organization={IEEE}
}

@inproceedings{mellinger2011minimum,
  title={Minimum snap trajectory generation and control for quadrotors},
  author={Mellinger, Daniel and Kumar, Vijay},
  booktitle={2011 IEEE intl. conf. on robotics and automation},
  pages={2520--2525},
  year={2011},
  organization={IEEE}
}

@inproceedings{richter2016polynomial,
  title={Polynomial trajectory planning for aggressive quadrotor flight in dense indoor environments},
  author={Richter, Charles and Bry, Adam and Roy, Nicholas},
  booktitle={Robotics Research: The 16th International Symposium ISRR},
  pages={649--666},
  year={2016},
  organization={Springer}
}

@inproceedings{shah2021ving,
  title={Ving: Learning open-world navigation with visual goals},
  author={Shah, Dhruv and Eysenbach, Benjamin and Kahn, Gregory and Rhinehart, Nicholas and Levine, Sergey},
  booktitle={2021 IEEE International Conference on Robotics and Automation (ICRA)},
  pages={13215--13222},
  year={2021},
  organization={IEEE}
}

@inproceedings{chen2024spatialvlm,
  title={Spatialvlm: Endowing vision-language models with spatial reasoning capabilities},
  author={Chen, Boyuan and Xu, Zhuo and Kirmani, Sean and Ichter, Brain and Sadigh, Dorsa and Guibas, Leonidas and Xia, Fei},
  booktitle={Proceedings of the IEEE/CVF Conference on Computer Vision and Pattern Recognition},
  pages={14455--14465},
  year={2024}
}

@article{yang2025qwen3,
  title={Qwen3 technical report},
  author={Yang, An and Li, Anfeng and Yang, Baosong and Zhang, Beichen and Hui, Binyuan and Zheng, Bo and Yu, Bowen and Gao, Chang and Huang, Chengen and Lv, Chenxu and others},
  journal={arXiv preprint arXiv:2505.09388},
  year={2025}
}

@manual{px4_ros2_user_guide,
  title={PX4 ROS 2 User Guide},
  year={2026},
  url={https://docs.px4.io/main/en/ros2/user_guide},
  note={Accessed: 2026-03-01}
}

@inproceedings{radford2023robust,
  title={Robust speech recognition via large-scale weak supervision},
  author={Radford, Alec and Kim, Jong Wook and Xu, Tao and Brockman, Greg and McLeavey, Christine and Sutskever, Ilya},
  booktitle={International conference on machine learning},
  pages={28492--28518},
  year={2023},
  organization={PMLR}
}

@inproceedings{tellex2011understanding,
  title={Understanding natural language commands for robotic navigation and mobile manipulation},
  author={Tellex, Stefanie and Kollar, Thomas and Dickerson, Steven and Walter, Matthew and Banerjee, Ashis and Teller, Seth and Roy, Nicholas},
  booktitle={Proc. of the AAAI Conf. on Artificial Intel.},
  volume={25},
  year={2011}
}

@inproceedings{anderson2018vision,
  title={Vision-and-language navigation: Interpreting visually-grounded navigation instructions in real environments},
  author={Anderson, Peter and Wu, Qi and Teney, Damien and Bruce, Jake and Johnson, Mark and S{\"u}nderhauf, Niko and Reid, Ian and Gould, Stephen and Van Den Hengel, Anton},
  booktitle={Proceedings of the IEEE conference on computer vision and pattern recognition},
  pages={3674--3683},
  year={2018}
}

@inproceedings{kuo2018deep,
  title={Deep sequential models for sampling-based planning},
  author={Kuo, Yen-Ling and Barbu, Andrei and Katz, Boris},
  booktitle={2018 IEEE/RSJ International Conference on Intelligent Robots and Systems (IROS)},
  year={2018},
  organization={IEEE}
}

@article{hu2019safe,
  title={Safe navigation with human instructions in complex scenes},
  author={Hu, Zhe and Pan, Jia and Fan, Tingxiang and Yang, Ruigang and Manocha, Dinesh},
  journal={IEEE Robotics and Automation Letters},
  year={2019},
  publisher={IEEE}
}

\end{document}